\documentclass{article}
\pdfoutput=1
\usepackage{colt10e}
\usepackage{times}
\usepackage{multirow}
\usepackage{graphicx}
\usepackage{subfigure}
\usepackage{natbib}
\usepackage[vlined,algoruled,linesnumbered]{algorithm2e}
\usepackage{tabularx}
\usepackage{afterpage }

\input{Definitions}

\title{Faster Rates for Training Max-Margin Markov Networks }

\author{Xinhua Zhang\\
  Dept.\ of Statistics\\
  Purdue University \\
  zhang305@stat.purdue.edu
  \And Ankan Saha\\
  Dept.\ of Computer Science\\
  University of Chicago \\
  ankans@cs.uchicago.edu
  \And S.V$\!.\,$N. Vishwanathan\\
  Dept.\ of Statistics and \\
  Dept.\ of Computer Science\\
  Purdue University\\
  vishy@stat.purdue.edu}

\begin{document}
\maketitle
\begin{abstract}
  Structured output prediction is an important machine learning problem
  both in theory and practice, and the max-margin Markov network (\mcn)
  is an effective approach.  All state-of-the-art algorithms for
  optimizing \mcn\ objectives take at least $O(1/\epsilon)$ number of
  iterations to find an $\epsilon$ accurate solution.  Recent results in
  structured optimization suggest that faster rates are possible by
  exploiting the structure of the objective function. Towards this end
  \cite{Nesterov05} proposed an excessive gap reduction technique based
  on Euclidean projections which converges in $O(1/\sqrt{\epsilon})$
  iterations on strongly convex functions. Unfortunately when applied to
  \mcn s, this approach does not admit graphical model factorization
  which, as in many existing algorithms, is crucial for keeping the cost
  per iteration tractable.  In this paper, we present a new excessive
  gap reduction technique based on Bregman projections which admits
  graphical model factorization naturally, and converges in
  $O(1/\sqrt{\epsilon})$ iterations. Compared with existing algorithms,
  the convergence rate of our method has better dependence on $\epsilon$
  and other parameters of the problem, and can be easily kernelized.
\end{abstract}

\section{Introduction}
\label{sec:Introduction}

In the supervised learning setting, one is given a training set of
labeled data points and the aim is to learn a function which predicts
labels on unseen data points. Sometimes the label space has a rich
internal structure which characterizes the combinatorial or recursive
inter-dependencies of the application domain. It is widely believed that
capturing these dependencies is critical for effectively learning with
\emph{structured output}. Examples of such problems include sequence
labeling, context free grammar parsing, and word alignment.  However,
parameter estimation is generally hard even for simple linear models,
because the size of the label space is potentially exponentially large
(see \eg\ \cite{BakHofSchSmoetal07}). Therefore it is crucial to
exploit the underlying conditional independence assumptions for the sake
of computational tractability. This is often done by defining a graphical
model on the output space, and exploiting the underlying graphical model
factorization to perform computations.

Research in structured prediction can broadly be categorized into two
tracks: Using a maximum a posterior estimate from the exponential family
results in conditional random fields \cite[CRFs,][]{LafMcCPer01}, and a
maximum margin approach leads to max-margin Markov networks
\cite[\mcn s,][]{TasGueKol04}.  Unsurprisingly, these two approaches
share many commonalities: First, they both minimize a regularized risk
with a square norm regularizer.  Second, they assume that there is a
joint feature map $\phivec$ which maps $(\xb, \yb)$ to a feature vector
in $\RR^{p}$.\footnote{We discuss kernels and associated feature maps
  into a Reproducing Kernel Hilbert Space (RKHS) in the appendix.}
Third, they assume a label loss $\ell(\yb, \yb^{i}; \xb^{i})$ which
quantifies the loss of predicting label $\yb$ when the correct label of
input $\xb^{i}$ is $\yb^{i}$. Finally, they assume that the space of
labels $\Ycal$ is endowed with a graphical model structure and that
$\phivec(\xb, \yb)$ and $\ell(\yb, \yb^{i}; \xb^{i})$ factorize
according to the cliques of this graphical model. The main difference is
in the loss function employed. CRFs minimize the $L_{2}$-regularized
logistic loss:
\begin{align}
  \label{eq:crf-objective}
  J(\wb) & = \frac{\lambda}{2} \nbr{\wb}^2 + \frac{1}{n} \sum_{i=1}^{n}
  \log \sum_{\yb \in \Ycal} \exp\rbr{\ell(\yb, \yb^{i}; \xb^{i})
      -\inner{\wb}{\phivec(\xb^{i}, \yb^{i}) - \phivec(\xb^{i}, \yb)}},
\end{align}
while the \mcn s minimize the $L_{2}$-regularized hinge loss
\begin{align}
  \label{eq:m3n-objective}
  J(\wb) & = \frac{\lambda}{2} \nbr{\wb}^2 + \frac{1}{n} \sum_{i=1}^{n}
  \max_{\yb \in \Ycal} \cbr{ \ell(\yb, \yb^{i}; \xb^{i}) -
    \inner{\wb}{\phivec(\xb^{i}, \yb^{i}) - \phivec(\xb^{i}, \yb)}  }.
\end{align}

\newpage
\begin{figure*}[t]
\begin{center}
\subfigure[Primal gap, dual gap, and duality gap]{
    \label{fig:duality_gap}
    \includegraphics[width=8.6cm]{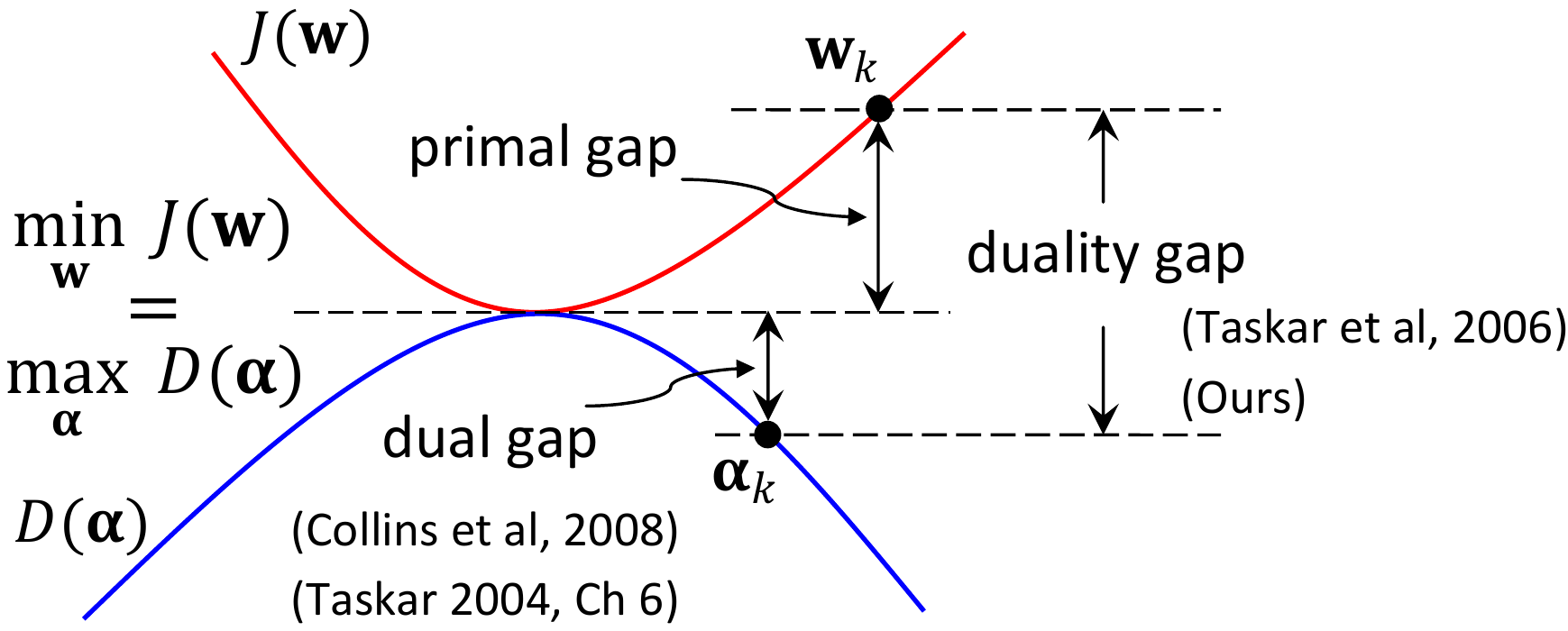}}
\subfigure[\BMRM\ gap (and similarly for \svmstruct)]{
    \label{fig:bmrm_gap}
    \includegraphics[width=6cm]{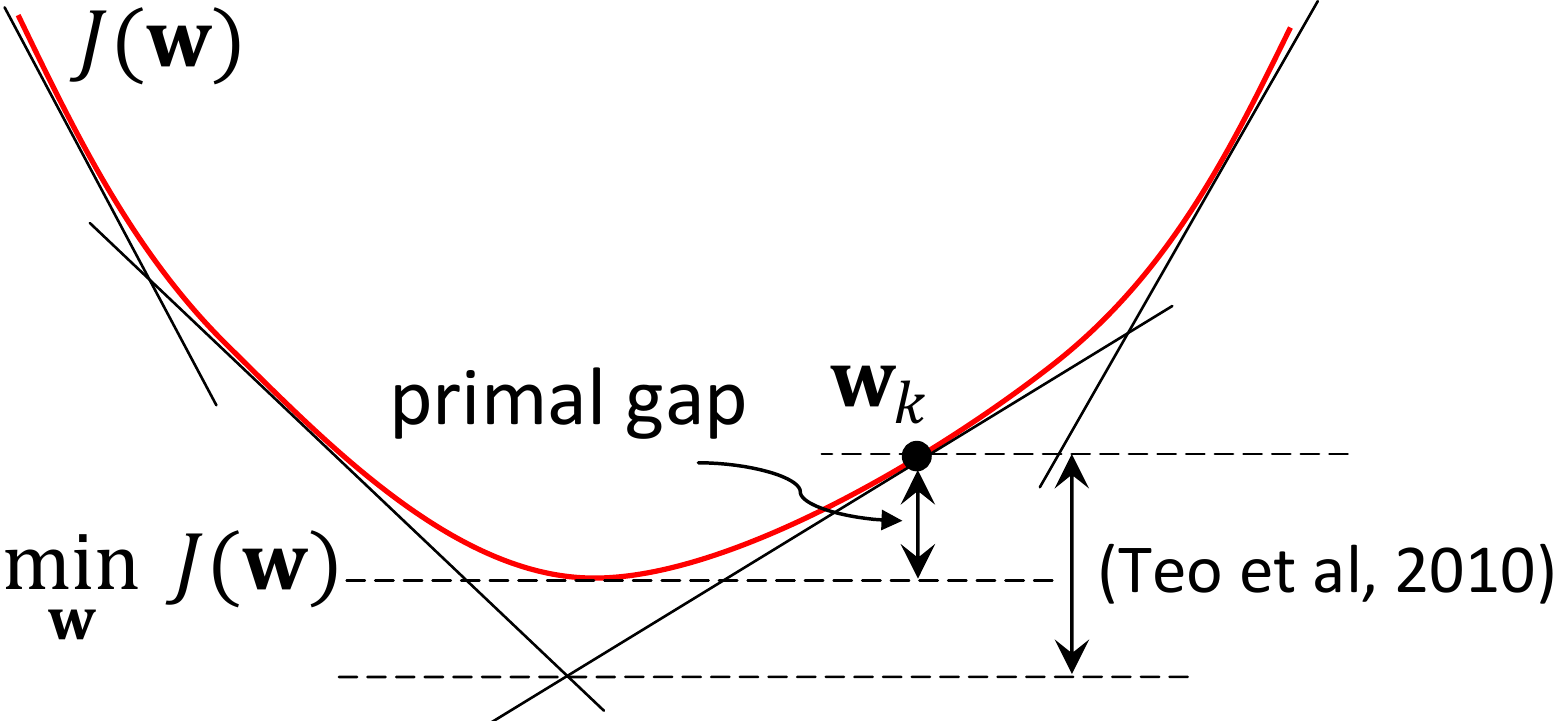}}
  \caption{Illustration of stopping criterion monitored by various
    algorithms; convergence rates are stated with respect to these
    stopping criterion. $D(\val)$ is the Lagrange dual of $J(\wb)$,
    and $\min_{\wb} J(\wb) = \max_{\val} D(\val)$. Neither the primal
    gap nor the dual gap is actually measurable in practice since
    $\min_{\wb} J(\wb)$ (and $\max_{\val} D(\val)$) is
    unknown. \BMRM\ (right) therefore uses a measurable upper bound of
    the primal gap.  \svmstruct\ monitors constraint violation, which
    can be also be translated to an upper bound on the primal gap.}
\label{fig:gaps}
\end{center}
\end{figure*}

\setlength{\extrarowheight}{2pt}

\begin{table}[t!]
  \centering
  \begin{tabularx}{\linewidth}{c|ccc|cc}
    \hline
    \multirow{2}{*}{Optimization algorithm} & \multirow{2}{*}{Primal/Dual} & \multirow{2}{*}{Type of gap} & Oracle & \multicolumn{2}{c}{Convergence rate} \\
    & & & for \mcn &  CRF & \mcn \\
    \hline  \hline
    \BMRM  & \multirow{2}{*}{primal} & \multirow{2}{*}{$\ge$primal gap} & \multirow{2}{*}{max} & \multirow{2}{*}{$O\rbr{\frac{1}{\lambda} \log \frac{1}{\epsilon}}$} & \multirow{2}{*}{$O\rbr{\frac{G^{2}}{\lambda \epsilon}}$} \\
    \cite{TeoVisSmoLe10} & & & & & \\
    \hline
    \svmstruct & \multirow{2}{*}{primal-dual} & constraint & \multirow{2}{*}{max} & \multirow{2}{*}{$n/a$} & \multirow{2}{*}{$O \rbr{\frac{G^{2}}{\lambda \epsilon^2}}$} \\
    \cite{TsoJoaHofAlt05} & & violation & & & \\
    \hline
    Extragradient & \multirow{2}{*}{primal-dual} & \multirow{2}{*}{duality gap} & \multirow{2}{*}{exp} & \multirow{2}{*}{$n/a$} & \multirow{2}{*}{$O\rbr{\frac{\log \abr{\Ycal}}{\epsilon}}$} \\
    \cite{TasLacJor06} & & & & & \\
    \hline
    Exponentiated gradient & \multirow{2}{*}{dual} & \multirow{2}{*}{dual gap} & \multirow{2}{*}{exp} & \multirow{2}{*}{$O\rbr{\frac{1}{\lambda} \log \frac{1}{\epsilon}}$} & \multirow{2}{*}{$O\rbr{\frac{\log \abr{\Ycal}}{\lambda \epsilon}}$} \\
    \cite{ColGloKooetal08} & & & & & \\
    \hline
    \smo\ & \multirow{2}{*}{dual} & \multirow{2}{*}{dual gap} & \multirow{2}{*}{max} & \multirow{2}{*}{$n/a$} & \multirow{2}{*}{$\ge O\rbr{\frac{n^{\ge 1}}{\lambda \epsilon}}$} \\
    \cite[][Chapter 6]{Taskar04} & & & & & \\
    \hline
    Our algorithm & primal-dual & duality gap & exp & $n/a$ & $O \rbr{\sqrt{\frac{\log \abr{\Ycal}}{\lambda \epsilon}}}$ \\
    \hline
  \end{tabularx}
  \caption{Comparison of specialized optimization algorithms for
    training structured prediction models. Primal-dual methods maintain
    estimation sequences in both primal and dual spaces.
    Details of the oracle will be discussed in
    Section~\ref{sec:Discussion}.
    The convergence rate highlights the dependence on both $\epsilon$ and
    some ``constants'' that are often hidden in the $O$ notation: $n$,
    $\lambda$, and the size of the label space $\abr{\Ycal}$. No formal
    convergence rate is known for \smo\ on \mcn, therefore we quote the
    best known rate for training
    binary SVMs due to \cite{LisSim09}.  The term $G$ in the convergence
    rate of \BMRM\ and \svmstruct\ denotes the maximum $L_{2}$ norm of
    the features vectors $\phivec(\xvec^i, \yvec)$.  The convergence rate of Extragradient depends on $\lambda$ in an indirectly way.
  }
  \label{tab:comp_algo}
\end{table}

A large body of literature exists on efficient algorithms for minimizing
the above objective functions. A summary of existing methods, and their
convergence rates (iterations needed to find an $\epsilon$ accurate
solution) can be found in Table \ref{tab:comp_algo}. The $\epsilon$
accuracy of a solution can be measured in many different ways.  As
Figure \ref{fig:gaps} depicts, different algorithms employ different but
somewhat related stopping criterion. This must be borne in mind when
interpreting the convergence rates in Table \ref{tab:comp_algo}.

Since \eqref{eq:crf-objective} is a smooth convex objective, classical
methods such as L-BFGS can directly be applied \cite{ShaPer03}.
Specialized solvers also exist. For instance a primal algorithm based on
bundle methods was proposed by \cite{TeoVisSmoLe10}, while a dual
algorithm for the same problem was proposed by
\cite{ColGloKooetal08}. Both algorithms converge at
$O(\frac{1}{\lambda}\log(1/\epsilon))$ rates to an $\epsilon$ accurate
solution, and, remarkably, their convergence rates are independent of
$n$ the number of data points, and $|\Ycal|$ the size of the label
space. It is widely believed in
optimization (see \eg\ Section 9.3 of \cite{BoyVan04}) that
unconstrained smooth strongly convex objective functions can be
minimized in $O(\log(1/\epsilon))$ iterations, and these specialized
optimizers also achieve this rate.

On the other hand, since \eqref{eq:m3n-objective} is a non-smooth convex
function, efficient algorithms are harder to come by. \svmstruct\ was
one of the first specialized algorithms to tackle this problem, and
\cite{TsoJoaHofAlt05} derived an $O(G^{2}/\lambda \epsilon^{2})$ rate
of convergence. Here $G$ denotes the maximum $L_{2}$ norm of the feature
vectors $\phivec(\xvec^i, \yvec)$. By refining their analysis,
\cite{TeoVisSmoLe10} proved a $O(G^{2}/\lambda \epsilon)$ rate of
convergence for a related but more general algorithm, which they called
bundle methods for regularized risk minimization (\BMRM).  At first glance, it looks
like the rates of convergence of these algorithms are independent of
$|\Ycal|$. This is somewhat misleading because, although the dependence
is not direct, the convergence rates depend on $G$, which is in turn
implicitly related to the size of $\Ycal$.

Optimization algorithms which solve \eqref{eq:m3n-objective} in the dual
have also been developed. For instance, the algorithm proposed by
\cite{ColGloKooetal08} performs exponentiated gradient descent in the
dual and converges at $O\rbr{\frac{\log \abr{\Ycal}}{\lambda \epsilon}}$
rates. Again, these rates of convergence are not surprising given the
well established lower bounds of \cite{NemYud83} who show that, in
general, non-smooth optimization problems cannot be solved in fewer than
$\Omega(1/\epsilon)$ iterations by solvers which treat the objective
function as a black box.

In this paper, we present an algorithm that provably converges to an
$\epsilon$ accurate solution of \eqref{eq:m3n-objective} in
$O\rbr{\sqrt{\frac{\log \abr{\Ycal}}{\lambda \epsilon}}}$
iterations. This does not contradict the lower bound because our
algorithm is not a general purpose black box optimizer. In fact, it
exploits the special form of the objective function
\eqref{eq:m3n-objective}.  Before launching into the technical details
we would like to highlight some important features of our
algorithm. First, compared to existing algorithms our convergence rates
are better in terms of $|\Ycal|$, $\lambda$, and $\epsilon$. Second, our
convergence analysis is tighter in that our rates are with respect to
the duality gap. Not only is the duality gap computable, it also upper
bounds the primal and dual gaps used by other algorithms (see Figure
\ref{fig:gaps}). Finally, our cost per iteration is comparable with
other algorithms.

To derive our algorithm we extend the recent excessive gap technique of
\cite{Nesterov05a} to Bregman projections and establish rates of
convergence (Section \ref{sec:Excessgaptechn}). This extension is
important because the original gradient based algorithm for strongly
convex objectives by \cite{Nesterov05a} does not admit graphical model
factorizations, which are crucial for efficiency in structured
prediction problems. We apply our resulting algorithm to the \mcn\
objective in Section \ref{sec:TrainMaximMarg}. A straightforward
implementation requires $O(|\Ycal|)$ computational complexities per iteration, which makes
it prohibitively expensive. We show that by exploiting the graphical
model structure of $\Ycal$ the cost per iteration can be reduced to
$O(\log|\Ycal|)$ (Section \ref{sec:EfficImplExpl}). Finally we contrast
our algorithm with existing techniques in Section
\ref{sec:Discussion}. The appendix contains some technical proofs and
details on how to handle kernels.

\section{Excessive Gap Technique with Bregman Projection}
\label{sec:Excessgaptechn}

The following three concepts from convex analysis are extensively used
in the sequel.  Define $\RRbar := \RR \cup \cbr{\infty}$.

\begin{definition}
  \label{def:strong-convex}
  A convex function $f:\RR^{n} \to \RRbar$ is strongly
  convex with respect to a norm $\|\cdot\|$ if there exists a constant
  $\rho > 0$ such that $f - \frac{\rho}{2} \|\cdot\|^{2}$ is convex.
  $\rho$ is called the modulus of strong convexity of $f$, and for
  brevity we will call $f$ $\rho$-strongly convex.
\end{definition}
\begin{definition}
  \label{def:lip-cont-grad}
  Suppose a function $f: \RR^n \to \RRbar$ is differentiable on $Q \subseteq
  \RR^n$.  Then $f$ is said to have Lipschitz continuous gradient (\lcg) with
  respect to a norm $\|\cdot\|$ if there exists a constant $L$ such that
  \begin{align}
    \label{eq:lip-cont-grad}
    \| \nabla f(\wb) - \nabla f(\wb')\| \leq L \| \wb - \wb'\| \qquad
    \forall\ \wb, \wb'\in Q.
  \end{align}
For brevity, we will call $f$ $L$-\lcg.
\end{definition}
\begin{definition}
  \label{def:fenchel_dual}
  The Fenchel dual of a function $f: \RR^n \to \RRbar$ is a function $f^{\star}:
  \RR^n \to \RRbar$ defined by
  \begin{align}
    \label{eq:fenchel-dual}
    f^{\star}(\wb^{\star}) = \sup_{\wb \in \RR^n}
    \cbr{\inner{\wb}{\wb^{\star}} - f(\wb)}
  \end{align}
\end{definition}
Strong convexity and \lcg\ are related by Fenchel duality according to
the following lemma:
\begin{lemma}[{\cite[][Theorem 4.2.1 and 4.2.2]{HirLem93}}]
$\phantom{.}$
\label{theorem:SC_LCG}
\begin{enumerate}
\item If $f: \RR^n \to \RRbar$ is $\rho$-strongly convex, then $f^{\star}$ is
  finite on $\RR^n$ and $f^{\star}$ is $\frac{1}{\rho}$-\lcg.
\item If $f: \RR^n \to \RR$ is convex, differentiable on $\RR^n$, and $L$-\lcg,
  then $f^{\star}$ is $\frac{1}{L}$-strongly convex.
\end{enumerate}
\end{lemma}

Let $Q_{1}$ and $Q_{2}$ be subsets of Euclidean spaces and $A$ be a linear map from $Q_1$ to $Q_2$.  Suppose $f$ and $g$ are convex functions defined on $Q_1$ and $Q_2$ respectively.  We are interested in the following optimization problem:
\begin{align}
  \label{eq:primal}
  \min_{\wb \in Q_{1}} J(\wb) \quad \text{ where } J(\wb) := f(\wb) +
  \gstar(A \wb) = f(\wb) + \max_{\val \in Q_2} \cbr{\inner{A
      \wb}{\val}-g(\val)}.
\end{align}
We will make the following standard assumptions: a) $Q_2$ is compact; b) with respect to a certain norm on $Q_1$, the function $f$ defined on $Q_1$ is
$\rho$-strongly convex but not necessarily \lcg, and c) with respect to a certain norm on $Q_2$, the function
$g$ defined on $Q_{2}$ is $L_{g}$-\lcg\ and convex, but not necessarily strongly
convex.  If we identify $f(\wb)$ with the regularizer and $\gstar(A \wb)$ with
the loss function, then it is clear that \eqref{eq:primal} has the same form as
\eqref{eq:crf-objective} and \eqref{eq:m3n-objective}. We will exploit this
observation in Section \ref{sec:TrainMaximMarg}.

The key difficulty in solving \eqref{eq:primal} arises because $\gstar$
and hence $J$ may potentially be non-smooth. Our aim is to uniformly
approximate $J(\wb)$ with a smooth and strongly convex function.
Towards this end let $d$ be a $\sigma$ strongly convex smooth function
with the following properties:
\begin{align*}
  \min_{\val \in Q_2} d(\val) = 0,
  \quad \val_{0} = \argmin_{\val \in Q_2} d(\val),
  \text{ and }  D := \max_{\val \in Q_2} d(\val).
\end{align*}
In optimization parlance, $d$ is called a prox-function. Let
$\mu \in \RR$ be an arbitrary positive constant, and
\begin{align}
  \label{eq:gstar-with-d2}
  (g + \mu d)^{\star}(\wb) = \sup_{\val \in Q_2}
  \cbr{\inner{\val}{\wb} - g(\val) - \mu \, d(\val)}.
\end{align}
If $D < \infty$ then it is easy to see that $(g + \mu \, d)^{\star}$ is
uniformly close to $\gstar$:
\begin{align}
  \label{eq:bounds-on-gstar}
  \gstar(\wb) - \mu D \leq (g + \mu d)^{\star}(\wb) \leq
  \gstar(\wb).
\end{align}
We will use $(g + \mu d)^{\star}$ to define a new objective function
\begin{align}
\label{eq:reg_primal}
  J_{\mu}(\wb) &:= f(\wb) + (g + \mu d)^{\star}(A\wb)
   = f(\wb) + \max_{\val \in Q_2} \cbr{\inner{A \wb}{\val}-g(\val) -
    \mu \, d(\val)}.
\end{align}

If some mild constraint qualifications hold \cite[\eg Theorem
3.3.5][]{BorLew00} one can write the dual $D(\val)$ of $J(\wb)$ using
$A^{\top}$ (the transpose of $A$) as
\begin{align}
  \label{eq:dual}
  D(\val) := -g(\val) - \fstar(-A^{\top} \val) = -g(\val) - \max_{\wb
    \in Q_{1}} \cbr{\inner{-A \wb}{\val} - f(\wb)},
\end{align}
and assert the following
\begin{align}
  \label{eq:borlew}
  \inf_{\wb \in Q_{1}} J(\wb) = \sup_{\val \in Q_{2}} D(\val), \quad
  \text{and} \quad J(\wb) \ge D(\val) \quad \forall\ \wb \in Q_{1}, \val
  \in Q_{2}.
\end{align}
The key idea of excessive gap minimization pioneered by
\cite{Nesterov05a} is to maintain two estimation sequences
$\cbr{\wb_{k}}$ and $\cbr{\val_{k}}$, together with a diminishing
sequence $\cbr{\mu_{k}}$ such that
\begin{equation}
  \label{eq:excessive_gap_condition}
  \boxed{J_{\mu_{k}}(\wb_{k}) \le D (\val_{k}), \text{ and } \lim_{k \to \infty} \mu_{k} = 0.}
\end{equation}
The idea is illustrated in Figure \ref{fig:ex_gap}.  In conjunction with
\eqref{eq:borlew} and \eqref{eq:bounds-on-gstar}, it is not hard to see
that $\cbr{\wb_{k}}$ and $\cbr{\val_{k}}$ approach the solution of
$\min_{\wb} J(\wb) = \max_{\val} D(\val)$. Using \eqref{eq:bounds-on-gstar},
\eqref{eq:reg_primal}, and \eqref{eq:excessive_gap_condition}, we can derive the rate of convergence of this
algorithm:
\begin{align}
  \label{eq:roc}
  J(\wb_{k}) - D(\val_{k}) \le J_{\mu_{k}}(\wb_{k}) + \mu_{k} D - D
  (\val_{k}) \le  \mu_{k} D.
\end{align}
In other words, the duality gap is reduced at the same rate
at which $\mu_{k}$ approaches $0$. All that remains to turn this idea
into an implementable algorithm is to answer the following two
questions:
\begin{enumerate}
\item How to efficiently find initial points $\wb_1$, $\val_{1}$ and
  $\mu_{1}$ that satisfy \eqref{eq:excessive_gap_condition}.
\item Given $\wb_{k}$, $\val_{k}$, and $\mu_{k}$, how to
  \emph{efficiently} find $\wb_{k+1}$, $\val_{k+1}$, and $\mu_{k+1}$
  which maintain \eqref{eq:excessive_gap_condition}.
\end{enumerate}

\begin{figure}[t]
\centering
    \includegraphics[width=4.5cm]{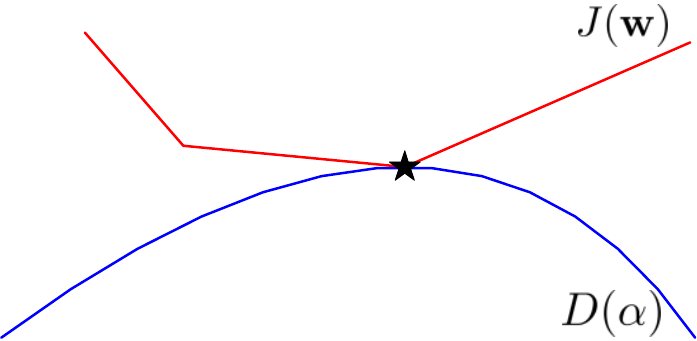}
    \hspace{0.3em}
    \includegraphics[width=5cm]{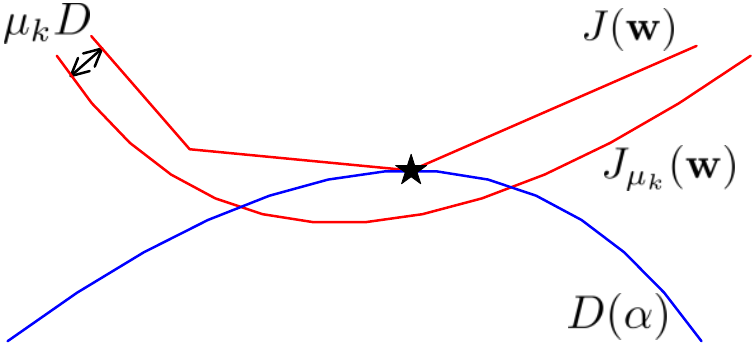}
    \hspace{0.3em}
    \includegraphics[width=5cm]{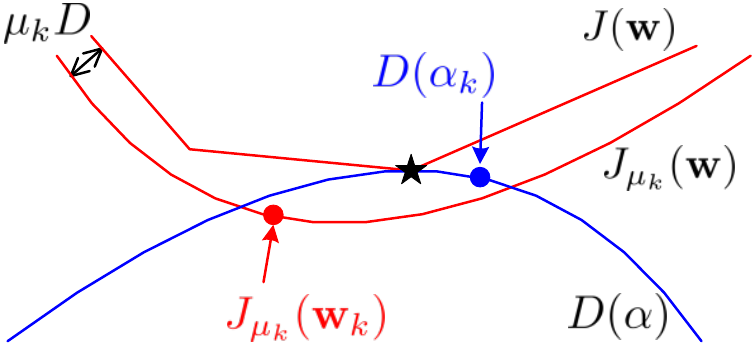}
\caption{Illustration of excessive gap.  When $\mu_k$ decreases to 0, the ``overlap" of $J_{\mu_k}(\wvec)$ and $D(\alphavec)$ becomes narrower and narrower.  And both $J_{\mu_k}(\wvec_k)$ and $D(\val_k)$ need to lie in this ``narrow tube".}
\label{fig:ex_gap}
\end{figure}

To achieve the best possible convergence rate it is desirable to anneal
$\mu_{k}$ as fast as possible while still allowing $\wb_{k}$ and
$\val_{k}$ to be updated efficiently. \cite{Nesterov05a} gave a
solution based on Euclidean projections, where $\mu_{k}$ decays at
$1/k^{2}$ rate and all updates can be computed in closed form.  We now
extend his ideas to updates based on Bregman
projections\footnote{\cite{Nesterov05a} did discuss updates based on
  Bregman projections, but just for the case where $f$ is convex rather
  than strongly convex.  Here, we show how to improve the convergence
  rate from $O(1/\epsilon)$ to $O(1/\sqrt{\epsilon})$ when $f$ is
  strongly convex.}, which will be the key to our application to
structured prediction problems later.  Since $d$ is differentiable, we
can define a Bregman divergence based on it:
\begin{align}
  \label{eq:bregman}
  \Delta(\valbar, \val) := d(\valbar) - d(\val) - \inner{\grad
    d(\val)}{\valbar - \val}.
\end{align}
Given a point $\val$ and a direction $\gvec$, we can define the Bregman
projection as:
\begin{align*}
  V(\val, \gvec) := \argmin_{\valbar \in Q_2}\{ \Delta(\valbar, \val) +
  \inner{\gvec}{\valbar - \val}\} = \argmin_{\valbar \in Q_2} d(\valbar)
  - \inner{\grad d(\val) - \gvec}{\valbar}.
\end{align*}

Since $f$ is assumed to be $\rho$-strongly convex, it follows from Lemma
\ref{theorem:SC_LCG} that $-D(\val)$ is \lcg.  If we denote its \lcg\ modulus
as $L$, then an easy calculation \cite[e.g.\ Eq.\ (7.2)][]{Nesterov05a}
shows that
\begin{align}
  \label{eq:lcg-of-D}
  L = \frac{\nbr{A}_{1,2}^2}{\rho} + L_g, \qquad \text{where}
  \nbr{A}_{1,2} := \max_{\nbr{\wb} = \nbr{\val} = 1} \inner{A
    \wb}{\val}.
\end{align}
For notational convenience, we define the following two maps:
\begin{subequations}
  \label{eq:alpha_and_w}
  \begin{align}
    \label{eq:alpha_to_w}
    \wb (\val) &:= \argmax_{\wb \in Q_1} \inner{-A \wb}{\val} - f(\wb) =
    \grad \fstar(-A^{\top} \val) \\
    \label{eq:w_to_alpha}
    \val_{\mu} (\wb) &:= \argmax_{\val \in Q_2} \cbr{\inner{A \wb}{\val} -
    g(\val) - \mu d(\val)} = \grad (g + \mu d)^{\star} (A \wb).
  \end{align}
\end{subequations}
Since both $f$ and $(g + \mu d)$ are strongly convex, the above maps are
unique and well defined.  With this notation in place we now describe
our excessive gap minimization method in Algorithm
\ref{algo:nesterov05_bregman}.
\begin{algorithm}[t]
  \caption{Excessive gap minimization}
  \label{algo:nesterov05_bregman}

  \KwIn{Function $f$ which is strongly convex, convex function $g$ which is \lcg.}
  \KwOut{Sequences $\cbr{\wb_{k}}$, $\cbr{\val_{k}}$, and $\cbr{\mu_{k}}$
    that satisfy \eqref{eq:excessive_gap_condition}, with
    $\lim_{k \to \infty} \mu_{k} = 0$.}

  Initialize: {Let $\val_0$ = minimizer of $d$ over $Q_2$, $\mu_{1} = \frac{L}{\sigma}$, $\wb_1 =
    \wb(\val_{0})$, $\val_{1} = V\rbr{\val_{0}, -\frac{1}{\mu_{1}}
      \grad D(\val_{0})}$.}\;

  \For{$k = 1, 2, \ldots$}{

    $\tau_{k} \leftarrow \frac{2}{k+3}$.

    $\valhat \leftarrow (1-\tau_{k}) \val_{k} + \tau_{k} \val_{\mu_{k}}(\wb_{k})$.

    $\wb_{k+1} \leftarrow (1-\tau_{k}) \wb_{k} + \tau_{k} \wb(\valhat)$.

    $\valtil \leftarrow V\rbr{\val_{\mu_{k}}(\wb_{k}),
      \frac{-\tau_{k}}{(1-\tau_{k}) \mu_{k}} \grad D(\valhat) }$.

    $\val_{k+1} \leftarrow (1-\tau_{k}) \val_{k} + \tau_{k} \valtil$.

    $\mu_{k+1} \leftarrow (1 - \tau_{k}) \mu_{k}$.
  }
\end{algorithm}
Unrolling the recursive update for $\mu_{k+1}$ yields
\begin{align}
  \label{eq:mu2-update}
  \mu_{k+1} = (1 - \tau_{k})\, \mu_{k} = \frac{k+1}{k+3} \,
  \mu_{k} = \frac{(k+1)(k) \ldots 2}{(k+3)(k+2)\ldots 4}
  \frac{L}{\sigma} = \frac{6}{(k+3)(k+2)} \frac{L}{\sigma}.
\end{align}
Plugging this into \eqref{eq:roc} and using \eqref{eq:lcg-of-D}
immediately yields a $O(1/\sqrt{\epsilon})$ rate of convergence of our
algorithm:
\begin{theorem}[Rate of convergence for duality gap]
\label{thm:rate_conv_duality_gap}
The sequences $\cbr{\wb_k}$ and $\cbr{\val_k}$ in Algorithm \ref{algo:nesterov05_bregman} satisfy
\begin{align}
\label{eq:convergence}
  J(\wb_{k}) - D(\val_{k}) &\le \frac{6LD}{\sigma (k+1) (k+2)} =
  \frac{6D}{\sigma (k+1) (k+2)} \rbr{\frac{\nbr{A}_{1,2}^2}{\rho} +
    L_g}.
\end{align}
\end{theorem}
All that remains is to show that
\begin{theorem}
  \label{thm:excessive_condi_satisfy}
  The update rule of Algorithm \ref{algo:nesterov05_bregman} guarantees
  that \eqref{eq:excessive_gap_condition} is satisfied for all $k \ge
  1$.
\end{theorem}
\begin{proof}
See Appendix \ref{sec:app_proof_exc_cond_meet}.
\end{proof}

When stated in terms of the dual gap (as opposed to the duality gap) our
convergence results can be strengthened slightly.
\begin{corollary}[Rate of convergence for dual gap]
  \label{thm:rate_conv_dual_gap}
  The sequence $\cbr{\val_k}$ in Algorithm
  \ref{algo:nesterov05_bregman} satisfy
  \begin{align}
    \label{eq:convergence_dual_gap}
    \max_{\val \in Q_2} D(\val) - D(\val_k) \le \frac{6 \ L \, d(\val^*)}{\sigma
      (k+1) (k+2)} = \frac{6 \ d(\val^*)}{\sigma (k+1) (k+2)}
    \rbr{\frac{\nbr{A}_{1,2}^2}{\rho} + L_g},
  \end{align}
where $\val^* := \argmax_{\val \in Q_2} D(\val)$.  Note $d(\val^*)$ is tighter than the $D$ in \eqref{eq:convergence}.
\end{corollary}
\begin{proof}
  See Appendix \ref{sec:app_dual_gap}.
\end{proof}

\section{Training Max-Margin Markov Networks}
\label{sec:TrainMaximMarg}

In the max-margin Markov network (\mcn) setting \cite{TasGueKol04},
we are given $n$ labeled data points $\cbr{\xb^{i}, \yb^{i}}_{i=1}^n$,
where $\xb^{i}$ are drawn from some space $\Xcal$ and $\yb^{i}$ belong
to some space $\Ycal$. We assume that there is a feature map $\phivec$
which maps $(\xb, \yb)$ to a feature vector in $\RR^{p}$. Furthermore,
for each $\xb^{i}$, there is a label loss $\ell^{i}_{\yb} := \ell(\yb,
\yb^{i}; \xb^{i})$ which quantifies the loss of predicting label $\yb$
when the correct label is $\yb^{i}$. Given this setup, the objective
function minimized by \mcn s can be written as
\begin{align}
  \label{eq:m3n_primal_obj}
  J(\wb) = \frac{\lambda}{2} \nbr{\wb}^{2} + \frac{1}{n} \sum_{i=1}^{n}
  \max_{\yb \in \Ycal} \cbr{ \ell^{i}_{\yb} -
    \inner{\wb}{\psivec^{i}_{\yb}} },
\end{align}
where we used the shorthand $\psivec^{i}_{\yb} := \phivec(\xb^{i},
\yb^{i}) - \phivec(\xb^{i}, \yb)$. To write \eqref{eq:m3n_primal_obj} in
the form of \eqref{eq:primal}, we define $Q_{1} = \RR^{p}$, $A$ to be a
$(n \abr{\Ycal})$-by-$p$ matrix whose $(i, \yb)$-th row is
$(-\psivec^{i}_{\yb})^{\top}$,
\begin{align*}
  f(\wb) = \frac{\lambda}{2} \nbr{\wb}^{2}_{2}, \quad \text{and} \quad
  \gstar(\ub) = \frac{1}{n} \sum_i \max_{\yb} \cbr{\ell^{i}_{\yb} +
    u^i_{\yb}}.
\end{align*}
Now, $g$ can be verified to be:
\begin{align}
  \label{eq:m3n-gdef}
  g(\val) = 
  \begin{cases}
    -\sum_i \sum_{\yb} \ell^{i}_{\yb} \alpha^{i}_{\yb} & \text{if }
    \alpha^{i}_{\yb} \ge 0, \text{ and } \sum_{\yb} \alpha^{i}_{\yb} =
    \frac{1}{n}, \ \forall\ i\\
    +\infty & \text{otherwise}.
  \end{cases}
\end{align}
The domain of $g$ is $Q_2 = \Scal^n := \cbr{\val \in [0, 1]^{n
    \abr{\Ycal}} : \sum_{\yb} \alpha^{i}_{\yb} = \frac{1}{n}, \ \forall\
  i}$, which is convex and compact. Using the $L_{2}$ norm on $Q_{1}$
(\ie, $\nbr{\wvec} = (\sum_i w_i^2)^{1/2}$), $f$ is clearly $\lambda$-strongly convex.  Similarly, if we use the $L_{1}$ norm on $Q_{2}$ (\ie,
$\nbr{\val} = \sum_{i} \sum_{\yvec} \abr{\alpha^i_{\yvec}}$), then $g$ is
0-\lcg.  By noting that $\fstar(-A^{\top} \val) = \frac{1}{2 \lambda}
\val^{\top} A A^{\top} \val$, one can write the dual form $D(\val) :
\Scal^n \mapsto \RR$ of $J(\wb)$ as
\begin{align}
  \label{eq:dval-m3n}
  D(\val) = -g(\val) - \fstar(-A^{\top} \val) = -\frac{1}{2\lambda}
  \val^{\top} A A^{\top} \val + \sum_i \sum_{\yb} \ell^{i}_{\yb}
  \alpha^{i}_{\yb},\qquad \val \in \Scal^{n}.
\end{align}

\subsection{Rates of Convergence}
\label{sec:RatesConvergence}

A natural prox-function to use in our setting is the relative entropy
with respect to the uniform distribution, which is defined as:
\begin{align}
  \label{eq:m3n-proxdef}
  d(\val) = \sum_{i=1}^{n} \sum_{\yb} \alpha^{i}_{\yb} \log \alpha^{i}_{\yb} +
  \log n + \log \abr{\Ycal},
\end{align}
The relative entropy is 1-strongly convex in
$\Scal^{n}$ with respect to the $L_{1}$ norm \cite[e.g.,][Proposition 5.1]{BecTeb03}. Furthermore, $d(\val) \leq D = \log
|\Ycal|$ for $\val \in \Scal^{n}$, and the norm of
$A$ can be computed via
\begin{align*}
  \nbr{A}_{1,2} = \max_{\wb \in \RR^p, \ub \in \RR^{n \abr{\Ycal}} }
  \cbr{\inner{A \wb}{\ub} : \sum_{i=1}^p w_i^2 = 1, \sum_{i=1}^{n}
    \sum_{\yb \in \Ycal} \abr{u^i_{\yb}} = 1} = \max_{i,\yb}
  \nbr{\psivec^{i}_{\yb}},
\end{align*}
where $\nbr{\psivec^{i}_{\yb}}$ is the Euclidean norm of $\psivec^{i}_{\yb}$.  Since $f$ is $\lambda$-strongly convex 
and $L_g = 0$, plugging this expression of $\nbr{A}_{1,2}$ into
\eqref{eq:convergence} and \eqref{eq:convergence_dual_gap}, we obtain
the following rates of convergence for our algorithm:
\begin{align*}
  J(\wb_{k}) - D(\val_{k}) \le \frac{6 \log \abr{\Ycal}}{(k+1) (k+2)}
  \frac{\max_{i,\yb}\nbr{\psivec^{i}_{\yb}}^2}{\lambda} \text{ and }
  \max_{\val \in Q_2} D(\val) - D(\val_{k}) \le \frac{6
    \text{KL}(\val^*||\val_0)}{(k+1) (k+2)}
  \frac{\max_{i,\yb}\nbr{\psivec^{i}_{\yb}}^2}{\lambda},
\end{align*}
where $\text{KL}(\val^*||\val_0)$ denotes the KL divergence between
$\val^{*}$ and the uniform distribution $\val_{0}$. Recall that for
distributions $\pb$ and $\qb$ the KL divergence is defined as
KL$(\pb||\qb) = \sum_i p_i \ln \frac{p_i}{q_i}$.

Therefore to reduce the
duality gap and dual gap below $\epsilon$, it suffices to take the
following number of steps respectively:
\begin{align}
  \label{eq:rate_conv_mcn}
  \text{Duality gap: } 2+\max_{i,\yb}\nbr{\psivec^{i}_{\yb}} \sqrt{\frac{6
      \log \abr{\Ycal}} {\lambda \epsilon}} \qquad \quad \text{Dual gap:
  } 2+\max_{i,\yb}\nbr{\psivec^{i}_{\yb}} \sqrt{\frac{6
      \text{KL}(\val^*||\val_0)} {\lambda \epsilon}}.
\end{align}

\subsection{Computing the Approximation $J_{\mu}(\wb)$ and Connection to CRFs}
\label{sec:CompAppr}

In this section we show how to compute $J_{\mu}(\wb)$. Towards this
end, we first compute $(g + \mu d)^{\star} (\ub)$.
\begin{lemma}
  \label{lem:dual-g-mud}
  The Fenchel dual of $(g + \mu d)$ is given by
  \begin{align}
    \label{eq:dual-g-mud}
    (g + \mu d)^{\star} (\ub) = \frac{\mu}{n} \sum_{i=1}^{n} \log
    \sum_{\yb \in \Ycal} \exp \rbr{\frac{u^{i}_{\yb} +
        \ell^{i}_{\yb}}{\mu}} - \mu \log \abr{\Ycal},
  \end{align}
  and the $(i, \yb)$-th element of its gradient can be written as
\begin{align}
  \label{eq:dual-g-mud-grad}
  \rbr{\grad (g + \mu d)^{\star} (\ub)}^{i}_{\yb} = \frac{1}{n}
  \exp \rbr{\frac{u^i_{\yb} + \ell^{i}_{\yb}}{\mu}} \left/ \sum_{\yb'}
    \exp \rbr{\frac{u^{i}_{\yb'} + \ell^{i}_{\yb'}}{\mu}}\right..
\end{align}
\end{lemma}
\begin{proof}
See Supplementary Material \ref{sec:proof_lem:dual-g-mud}.
\end{proof}

Using the above lemma, plugging in the definition of $A$ and
$\psivec^{i}_{\yb}$, and assuming that $\ell^{i}_{\yb^{i}} = 0$, we get
\begin{align}
  \label{eq:smooth-primal}
  J_{\mu}(\wb) = f(\wb) + (g+\mu d)^{\star}(A\wb) &= \frac{\lambda}{2}
  \nbr{\wb}^{2}_{2} - \frac{\mu}{n} \sum_{i=1}^{n} \log p(\yb^{i} |
  \xb^{i}; \wb) - \mu \log \abr{\Ycal}, \\
  \text{where} \qquad \qquad p(\yb | \xb^{i}; \wb) &\propto \exp \rbr{\frac{\ell^{i}_{\yb} +
      \inner{\wb}{\phivec(\xb^{i}, \yb)}}{\mu}} \nonumber.
\end{align}
This interpretation clearly shows that the approximation
$J_{\mu}(\wb)$ essentially converts the maximum margin estimation
problem \eqref{eq:m3n-objective} into a CRF estimation problem
\eqref{eq:crf-objective}. Here $\mu$ determines the quality of the
approximation; when $\mu \to 0$, $p(\yb | \xb^{i}; \wb)$ tends to the
delta distribution with the probability mass concentrated on
$\argmax_{\yvec} \ell^i_{\yvec} + \inner{\wb}{\phivec(\xvec^i,
  \yvec)}$.  Besides, the loss $\ell^i_{\yvec}$ rescales the
distribution.

Given the above interpretation, it is tempting to argue that every
non-smooth problem can be solved by computing a smooth approximation
$J_{\mu}(\wb)$, and applying a standard smooth convex optimizer to
minimize $J_{\mu}(\wb)$. Unfortunately, this approach is fraught with
problems. In order to get a close enough approximation of $J(\wb)$ the
$\mu$ needs to be set to a very small number which makes
$J_{\mu}(\wb)$ ill-conditioned and leads to numerical issues in
the optimizer. The excessive gap technique adaptively changes
the $\mu$ in each iteration in order to avoid these problems.

\section{Efficient Implementation by Exploiting Clique Decomposition}
\label{sec:EfficImplExpl}

In the structured large margin setting, the number of labels $|\Ycal|$
could potentially be exponentially large.  For example, if a sequence
has $l$ nodes and each node has two states, then $\abr{\Ycal} = 2^l$.
A naive implementation of the excessive gap reduction algorithm
described in the previous section requires maintaining and updating
$O(|\Ycal|)$ coefficients at every iteration, which is prohibitively
expensive.  With a view to reducing the computational complexity, and
also to take into account the inherent conditional independence properties
of the output space, it is customary to assume that $\Ycal$ is endowed
with a graphical model structure; we refer the reader to
\cite{BakHofSchSmoetal07} for an in-depth treatment of this issue.
For our purposes it suffices to assume that $\ell(\yb, \yb^{i};
\xb^{i})$ and $\phivec(\xb^{i}, \yb)$ decompose according to the
cliques\footnote{Any fully connected subgraph of a graph is called a clique.} of an undirected graphical model, and hence can be written
(with some abuse of notation) as
\begin{align}
  \label{eq:decomp_m3n}
  \ell^{i}_{\yb} = \ell(\yb,\yb^{i}; \xb^{i}) = \sum_{c \in \Ccal}
  \ell(y_{c}, y_{c}^{i}; \xb^{i}) = \sum_{c \in \Ccal}
  \ell^{i}_{y_{c}}, \quad \phivec(\xb^{i}, \yb) = \mathop{\oplus}\limits_{c
    \in \Ccal} \phivec(\xb^{i}, y_{c}), \text{ and } \psivec^{i}_{\yb} =
  \mathop{\oplus}\limits_{c \in \Ccal} \psivec^{i}_{y_{c}}.
\end{align}
Here $\Ccal$ denotes the set of all cliques of the graphical model and
$\oplus$ denotes vector concatenation. More explicitly,
$\psivec^{i}_{\yb}$ is the vector on the graphical model obtained by
accumulating the vector $\psivec^{i}_{y_{c}}$ on all the cliques $c$ of the graph.

Let $h_{c}(y_c)$ be an arbitrary real valued function on the value of
$\yvec$ restricted to clique $c$.  Graphical models define a
distribution $p (\yb)$ on $\yvec \in \Ycal$ whose density takes the
following factorized form:
\begin{align}
  \label{eq:factorized-def}
  p (\yb) \propto q(\yvec) =  \prod_{c \in \Ccal} \exp\rbr{h_c(y_c)}.
\end{align}
The key advantage of a graphical model is that the marginals on the
cliques can be efficiently computed:
\begin{align*}
  m_{y_c} := \sum_{\zvec: \zvec|_c = y_c} q(\zvec) = \sum_{\zvec: \zvec|_c = y_c} \prod_{c' \in \Ccal} \exp\rbr{h_{c'}(z_{c'})}.
\end{align*}
where the summation is over all the configurations $\zvec$ in $\Ycal$
whose restriction on the clique $c$ equals $y_c$.  Although $\Ycal$ can
be exponentially large, efficient dynamic programming algorithms exist
that exploit the factorized form \eqref{eq:factorized-def}, \eg\ belief
propagation \cite{Lauritzen96}.  The computational cost is
$O(s^\omega)$ where $s$ is the number of states of each node, and
$\omega$ is the maximum size of the cliques.  For example, a linear chain has $\omega = 2$.  When $\omega$ is large,
approximate algorithms also exist \cite{WaiJor08, AndFreDouJor03, KscFreLoe01}.
In the sequel we will assume that our graphical models are tractable, \ie, $\omega$ is low.

\subsection{Basics}
\label{sec:Basics}

At each iteration of Algorithm \ref{algo:nesterov05_bregman}, we need
to compute four quantities: $\wb(\val)$, $\grad D(\val)$,
$\val_{\mu}(\wb)$, and $V(\val, \gvec)$.  Below we rewrite them by
taking into account the factorization \eqref{eq:decomp_m3n}, and
postpone to Section \ref{sec:EfficientComputation} the discussion on
how to compute them efficiently.  Since $\alpha^{i}_{\yb} \geq 0$ and
$\sum_{\yb} \alpha^{i}_{\yb} = \frac{1}{n}$, the
$\cbr{\alpha^{i}_{\yb} : \yb \in \Ycal}$ form an unnormalized
distribution, and we denote its (unnormalized) marginal distribution
on clique $c$ by
\begin{align}
\label{eq:marginals_alpha}
  \alpha^{i}_{y_{c}} := \sum\nolimits_{\zvec: \zvec |_c = y_c} \alpha^i_{\zvec}.
\end{align}
The feature expectations on the cliques with respect to the
unnormalized distributions $\val$ are important:
\begin{align}
  \label{eq:unnormalized_dist}
  \FF\sbr{\psivec^{i}_{y_{c}};\val} := \sum_{y_{c}} \alpha^{i}_{y_{c}}
  \psivec^{i}_{y_{c}}, \quad \text{and} \quad \FF[\psivec_c; \val] :=
  \sum_i \FF\sbr{\psivec^{i}_{y_{c}};\val}.
\end{align}
Clearly, if for all $i$ the marginals of $\val$ on the cliques (\ie, $\cbr{\alpha^{i}_{y_{c}} : i, c, y_c}$ in \eqref{eq:marginals_alpha}) are available, then these two expectations can be computed efficiently.
\begin{itemize}
\item \textbf{$\wb(\val)$:} As a consequence of \eqref{eq:decomp_m3n} we
  can write $\psivec^{i}_{\yb} = \mathop{\oplus}\limits_{c \in \Ccal}
  \psivec^{i}_{y_{c}}$. Plugging this into \eqref{eq:alpha_to_w} and
  recalling that $\grad \fstar(-A^{\top} \val) = \frac{-1}{\lambda}
  A^{\top} \val$ yields the following expression for $\wb(\val) =
  \frac{-1}{\lambda} A^{\top} \val$:
  \begin{eqnarray}
    \label{eq:wb_val}
    \wb(\val) \! = \! \frac{1}{\lambda} \sum_i \sum_{\yb} \alpha^{i}_{\yb}
    \psivec^{i}_{\yb} \! = \! \frac{1}{\lambda} \sum_{i} \sum_{\yb} \alpha^{i}_{\yb} \! \rbr{\mathop{\oplus} \limits_{c \in \Ccal} \psivec^{i}_{y_{c}}} \! = \! \frac{1}{\lambda}
    \mathop{\oplus}\limits_{c \in \Ccal} \! \rbr{\sum_i
      \FF\sbr{\psivec^{i}_{y_{c}};\val}} \! = \! \frac{1}{\lambda} \mathop{\oplus}\limits_{c \in
      \Ccal} \FF[\psivec_c; \val].
  \end{eqnarray}
\item \textbf{$\grad D(\val)$:} Using \eqref{eq:dval-m3n} and the
  definition of $\wb(\val)$, the $(i,\yb)$-th element of $\grad D(\val)$ can be written as
  \begin{align}
    \label{eq:grad_D_alpha}
    \rbr{\grad D(\val)}^{i}_{\yb} &= \ell^{i}_{\yb} -
    \frac{1}{\lambda} \rbr{AA^{\top} \val}^{i}_{\yb} = \ell^{i}_{\yb} -
    \inner{\psivec^{i}_{\yb}}{\wb(\val)} = \sum_{c} \rbr{ \ell^{i}_{y_{c}} - \frac{1}{\lambda} \inner{\psivec^{i}_{y_{c}}}{\FF[\psivec_c; \val]}}.
  \end{align}
\item \textbf{$\val_{\mu}(\wb)$:} Using \eqref{eq:w_to_alpha} and
  \eqref{eq:dual-g-mud-grad}, the $(i,\yb)$-th element of
  $\val_{\mu}(\wb)$ given by $\rbr{\grad (g + \mu d)^{\star}
    (A\wb)}^{i}_{\yb}$ can be written as
  \begin{align}
    \label{eq:alpha_mu_k_wb_k}
    \rbr{\val_{\mu}(\wb)}^{i}_{\yb} = \frac{1}{n} \frac{\exp
      \rbr{\mu^{-1}\rbr{\ell^{i}_{\yb}-\inner{\psivec^{i}_{\yb}}{\wb}}}}{
      \sum_{\yb'} \exp \rbr{\mu^{-1}\rbr{\ell^{i}_{\yb'} -
          \inner{\psivec^{i}_{\yb'}}{\wb}}}} = \frac{1}{n}
    \frac{\prod_{c} \exp
      \rbr{\mu^{-1}\rbr{\ell^{i}_{y_{c}}-\inner{\psivec^{i}_{y_{c}}}{\wb_c}}}}{\sum_{\yb'}
      \prod_{c} \exp
      \rbr{\mu^{-1}\rbr{\ell^{i}_{y'_{c}}-\inner{\psivec^{i}_{y'_{c}}}{\wb_{c}}}}}.
  \end{align}
  
  \vspace{-1.3em}
\item \textbf{$V(\val, \gvec)$:} Since the prox-function $d$ is the relative entropy, the $(i,\yb)$-th element of $V(\val, \gvec)$ is
  \begin{align}
    \label{eq:V_alpha_g}
    \rbr{V(\val, \gvec)}^{i}_{\yb} = \frac{1}{n} \frac{\alpha^{i}_{\yb}
      \exp(-g^{i}_{\yb})}{\sum_{\yb'} \alpha^{i}_{\yb'}
      \exp(-g^{i}_{\yb'})}.
  \end{align}
\end{itemize}

\subsection{Efficient Computation}
\label{sec:EfficientComputation}

We now show how the algorithm can be made efficient by taking into
account \eqref{eq:decomp_m3n}. Key to our efficient implementation are
the following four observations from Algorithm
\ref{algo:nesterov05_bregman} when applied to the structured large
margin setting.  In particular, we will exploit the fact that the
marginals of $\val_k$ can be updated iteratively.
\begin{itemize}
\item \textbf{The marginals of $\val_{\mu_k}(\wb_k)$ and $\valhat$ can be computed efficiently}. From \eqref{eq:alpha_mu_k_wb_k} it is easy to see that $\val_{\mu_k}(\wb_k)$ can be written as a product of factors over cliques, that is, in the form of \eqref{eq:factorized-def}. Therefore, the marginals of $\val_{\mu_k}(\wb_k)$ can be computed efficiently.  As a result, if we keep track of the marginal distributions of $\val_k$, then it is trivial to compute the marginals of $\valhat = (1-\tau_{k}) \val_{k} + \tau_{k} \val_{\mu_{k}}(\wb_{k})$.
\item \textbf{The marginals of $\valtil$ can be computed efficiently}. Define $\eta = \frac{-\tau_{k}}{(1-\tau_{k})\mu_{k}}$. By plugging in \eqref{eq:grad_D_alpha} and \eqref{eq:alpha_mu_k_wb_k} into
  \eqref{eq:V_alpha_g} and observing that $\grad D(\val)$ can be written as a sum of terms over cliques obtains:
  \begin{align}
    \nonumber \valtil^{i}_{\yb} & =
    \rbr{V\rbr{\val_{\mu_{k}}(\wb_{k}), \eta \grad
        D(\valhat)}}^{i}_{\yb} \propto \rbr{\val_{\mu_{k}}(\wb_{k})}^{i}_{\yb}
    \exp\rbr{-\eta \rbr{\grad D(\valhat)}^{i}_{\yb}}
    \\
    \label{eq:valtil-factor}
    & = \prod_{c} \exp
    \rbr{\mu_{k}^{-1}\rbr{\ell^{i}_{y_{c}}-\inner{\psivec^{i}_{y_{c}}}{(\wb_{k})_{c}}}
      - \eta \ell^{i}_{y_{c}} + \eta \lambda^{-1}
      \inner{\psivec^{i}_{y_{c}}}{\FF[\psivec_c; \valhat]}}.
  \end{align}
  Clearly, $\valtil$ factorizes and has the form of \eqref{eq:factorized-def}.  Hence its marginals can be computed efficiently.
\item  \textbf{The marginals of $\val_k$ can be updated efficiently}.
  Given the marginals of $\valtil$, it is trivial to update the marginals of $\val_{k+1}$ since $\val_{k+1} = (1-\tau_{k}) \val_{k} + \tau_{k} \valtil$.  For convenience, define $\val_c := \cbr{\alpha^i_{y_c}: i, y_c}$.
\item \textbf{$\wb_{k}$ can be updated efficiently}.  According to step 5 of Algorithm \ref{algo:nesterov05_bregman}, by using \eqref{eq:wb_val} we have
    \begin{align*}
      (\wb_{k+1})_c = (1-\tau_{k}) (\wb_{k})_c + \tau_{k} (\wb(\valhat))_c = (1-\tau_k) (\wb_{k})_c + \tau_k \lambda^{-1} \FF[\psivec_c; \valhat].
    \end{align*}
\end{itemize}

Leveraging these observations, Algorithm
\ref{algo:nesterov05_bregman_struct} provides a complete listing of how
to implement the excessive gap technique with Bregman projections for
training \mcn. It focuses on clarifying the ideas; a practical
implementation can be sped up in many ways.  The last issue to be
addressed is the computation of the primal and dual objectives
$J(\wb_k)$ and $D(\val_k)$, so as to monitor the duality gap. See
Appendix \ref{sec:primal_dual_eval} for details.

\begin{algorithm}[t]
  \caption{Max-margin structured learning using clique factorization}
  \label{algo:nesterov05_bregman_struct}

  \KwIn{Loss functions $\cbr{\ell^i_\yb}$ and features
    $\cbr{\psivec^i_\yb}$, a regularization parameter $\lambda$, a tolerance level $\epsilon > 0$.}
  \KwOut{A pair $\wb$ and $\val$ that satisfy $J(\wb) - D(\val) < \epsilon$.}

  Initialize: $k  \leftarrow 1$, $\mu_{1}  \leftarrow  \frac{1}{\lambda} \max_{i, \yvec} \nbr{\psivec^i_{\yvec}}^2$, $\val_0  \leftarrow \rbr {\frac{1}{n|\Ycal|},\hdots, \frac{1} {n|\Ycal|}}^{\top} \in \RR^{n|\Ycal|}$.\;
  Update $\wb_1  \leftarrow  \wb(\val_{0}) = \frac{1}{\lambda} \mathop{\oplus}\nolimits_{c \in \Ccal} \FF[\psivec_c; \val_0]$, $\val_{1}   \leftarrow  V  \rbr{\val_{0}, -\frac{1}{\mu_{1}} \grad D(\val_{0})}$ and compute its marginals.\;

  \While($\qquad$/* Termination criteria: duality gap falls below $\epsilon$ */){$J(\wb_k) - D(\val_k) \ge \epsilon$} {
    $\tau_{k} \leftarrow \frac{2}{k+3}$. \;
    Compute the marginals of $\val_{\mu_k}(\wb_k)$ by exploiting
    \eqref{eq:alpha_mu_k_wb_k}. \;
    \ForAll{\emph{cliques} $c \in \Ccal$}{
      Compute the marginals $\valhat_c$ by convex combination: $\valhat_c \leftarrow (1-\tau_k) (\val_k)_c + \tau_k (\val_{\mu_k}(\wb_k))_c$.\;
      Update the weight on clique $c$: $\rbr{\wb_{k+1}}_{c} \leftarrow \rbr{1-\tau_{k}}\rbr{\wb_{k}}_{c} + \frac{\tau_{k}}{\lambda} \sum_i \FF\sbr{\psivec^{i}_{y_{c}}; \valhat_c} $. \; 
    }
    Compute the marginals of $\valtil$ by exploiting    \eqref{eq:valtil-factor} and using the marginals $\cbr{\valhat_c}$. \;
      \ForAll{\emph{cliques} $c \in \Ccal$}{
        Update the marginals $(\val_k)_c$ by convex combination: $(\val_{k+1})_c \leftarrow (1-\tau_k) (\val_k)_c + \tau_k \valtil_c$.
      }
      Update 
      $\mu_{k+1} \leftarrow (1 - \tau_{k}) \mu_{k}$, $k \leftarrow k + 1$.\;
  }
  \KwRet{$\wb_k$ and $\val_k$.}\;
\end{algorithm}

\subsection{Kernelization}

When nonlinear kernels are used, the feature vectors
$\phivec^i_{\yvec}$ are not expressed explicitly and only their inner
products can be evaluated via kernels on the cliques:
\begin{align*}
\inner{\psivec^i_{\yvec}}{\psivec^j_{\yvec'}} \! := k((\xvec^i,
\yvec),(\xvec^j, \yvec')) \! = \! \sum_c k_c((\xvec^i, y_c), (\xvec^j,
y'_c)), \ \ \text{where } \ k_c((\xvec^i, y_c), (\xvec^j, y'_c)) \! :=
\! \inner{ \psivec^i_{y_c}}{\psivec^j_{y'_c}}.
\end{align*}
Algorithm \ref{algo:nesterov05_bregman_struct} is no longer applicable
because no explicit expression of $\wb$ is available.  However, by
rewriting $\wb_k$ as the feature expectations with respect to some
underlying distribution which can be updated implicitly, all the updates
and objective function evaluations can still be done efficiently.
Details are in Appendix \ref{sec:app_kernel}.

\subsection{Efficiency in Memory and Computation}

For concreteness, let us consider a sequence as an example. Here the
cliques are just edges between consecutive nodes.  Suppose there are
$l+1$ nodes and each node has $s$ states.  The memory cost of
Algorithm \ref{algo:nesterov05_bregman_struct} is $O(nls^2)$, due to
the storage of the marginals.  The computational cost per iteration is
dominated by calculating the marginals of $\valhat$ and $\valtil$, which is
$O(nls^2)$ by standard graphical model inference.  The rest operations
in Algorithm \ref{algo:nesterov05_bregman_struct} cost $O(nls^2)$ for
linear kernels.  If nonlinear kernels are used, then the cost becomes
$O(n^2ls^2)$ (see Appendix \ref{sec:app_kernel}).

\section{Discussion}
\label{sec:Discussion}

Structured output prediction is an important learning task in both
theory and practice.  The main contribution of our paper is two fold.
First, we identified an efficient algorithm by \cite{Nesterov05a} for
solving the optimization problems in structured prediction.  We proved
the $O(1/\sqrt{\epsilon})$ rate of convergence for the Bregman
projection based updates in excessive gap optimization, while
\cite{Nesterov05a} showed this rate only for projected gradient style
updates. In \mcn\ optimization, Bregman projection plays a key role in
factorizing the computations, while technically such factorizations
are not applicable to projected gradient.  Second, we designed a
nontrivial application of the excessive gap technique to
\mcn\ optimization, in which the computations are kept efficient by
using the graphical model decomposition.  Kernelized objectives can
also be handled by our method, and we proved superior convergence and
computational guarantees than existing algorithms.


When \mcn s are trained in a batch fashion, we can compare the
convergence rate of dual gap between our algorithm and the
exponentiated gradient method \cite[\expgrad,][]{ColGloKooetal08}.
Assume $\val_0$, the initial value of $\val$, is the uniform
distribution and $\val^*$ is the optimal dual solution.  Then by
\eqref{eq:rate_conv_mcn}, we have
\begin{align*}
  \text{Ours:} \quad \max_{i,\yb}\nbr{\psivec^{i}_{\yb}} \sqrt{\frac{6
      \text{KL}(\val^* || \val_{0})} {\lambda \epsilon}}, \qquad
  \qquad \text{\expgrad:} \quad \max_{i,\yb}\nbr{\psivec^{i}_{\yb}}^2
  \frac{\text{KL}(\val^* || \val_{0})}{\lambda \epsilon}.
\end{align*}
It is clear that our iteration bound is almost the square root of
\expgrad, and has much better dependence on $\epsilon$, $\lambda$,
$\max_{i,\yb}\nbr{\psivec^{i}_{\yb}}$, as well as the divergence from
the initial guess to the optimal solution $\text{KL}(\val^* ||
\val_{0})$.

In addition, the cost per iteration of our algorithm is almost the
same as \expgrad, and both are governed by the computation of the
expected feature values on the cliques (which we call exp-oracle), or
equivalently the marginal distributions.  For graphical models, exact
inference algorithms such as belief propagation can compute the
marginals via dynamic programming \cite{Lauritzen96}.  Finally,
although both algorithms require marginalization, they are calculated
in very different ways.  In \expgrad, the dual variables $\val$
correspond to a factorized distribution, and in each iteration its
potential functions on the cliques are updated using the exponentiated
gradient rule.  In contrast, our algorithm explicitly updates the
marginal distributions of $\val_k$ on the cliques, and marginalization
inference is needed only for $\valhat$ and $\valtil$.  Indeed, the
joint distribution $\val$ does \emph{not} factorize, which can be seen
from step 7 of Algorithm \ref{algo:nesterov05_bregman}: the convex
combination of two factorized distributions is not necessarily
factorized.

Marginalization is just one type of query that can be answered
efficiently by graphical models, and another important query is the
max a-posteriori inference (which we call max-oracle): given the
current model $\wb$, find the $\argmax$ in \eqref{eq:m3n-objective}.
Max-oracle has been used by greedy algorithms such as cutting plane
(\BMRM\ and \svmstruct) and sequential minimal optimization
\cite[\smo,][Chapter 6]{Taskar04}.  \smo\ picks the steepest descent
coordinate in the dual and greedily optimizes the quadratic
analytically, but its convergence rate is slower than \BMRM\ by a
factor $n$.  The max-oracle again relies on graphical models for
dynamical programming \cite{KscFreLoe01}, and many existing
combinatorial optimizers can also be used, such as in the applications
of matching \cite{TasLacKle05} and context free grammar parsing
\cite{TasKleColKoletal04}.  Furthermore, this oracle is
particularly useful for solving the slack rescaling variant of \mcn\
proposed by \cite{TsoJoaHofAlt05}:
\begin{align}
\label{eq:slack-objective}
    J(\wb) & = \frac{\lambda}{2} \nbr{\wb}^2 + \frac{1}{n} \sum_{i=1}^{n}
    \max_{\yb \in \Ycal} \cbr{ \ell(\yb, \yb^{i}; \xb^{i}) \rbr{1-
        \inner{\wb}{\phivec(\xb^{i}, \yb^{i}) - \phivec(\xb^{i}, \yb)}}}.
\end{align}
Here two factorized terms get multiplied, which causes additional
complexity in finding the maximizer.  \cite[][Section
1.4.1]{AltHofTso07} solved this problem by a modified dynamic
program.  Nevertheless, it is not clear how \expgrad\ or our method
can be used to optimize this objective.

In the quest for faster optimization algorithms for \mcn s, the
following three questions are important: how hard is it to optimize
\mcn\ intrinsically, how informative is the oracle which is the only way
for the algorithm to access the objective function, and how well does
the algorithm make use of such information.  The superiority of our
algorithm suggests that the exp-oracle is more informative than the max-oracle, and a deeper explanation is that the max-oracle is local while
the exp-oracle is not \cite[][Section 1.3]{NemYud83}.  Hence there is
no surprise that the less informative max-oracle is easier to compute,
which makes it applicable to a wider range of problems such as
\eqref{eq:slack-objective}.  Moreover, the comparison between \expgrad\
and our algorithm shows that even if the exp oracle is used, the
algorithm still needs to make good use of it in order to converge
faster.

For future research, it is interesting to study the lower bound
complexity for optimizing \mcn, including the dependence on $\epsilon$,
$n$, $\lambda$, $\Ycal$, and probably even on the graphical model
topology.  Empirical evaluation of our algorithm is also desirable,
along the lines of sequence labeling, word alignment, context free
grammar parsing, etc.

\bibliographystyle{plain}

{\small
  \bibliography{EG_Updates}
}

\appendix
\section*{Appendix (to be considered in the 13 page limit)}

\section{Proof of Theorem \ref{thm:excessive_condi_satisfy}}
\label{sec:app_proof_exc_cond_meet}

To prove Theorem \ref{thm:excessive_condi_satisfy}, we begin with a technical lemma.
\begin{lemma}
\label{lamma:nesterov:helper_alpha}
(Lemma 7.2 of \cite{Nesterov05a})
  For any $\val$ and $\valbar$, we have
  \begin{align*}
    D(\val) + \inner{\grad D(\val)}{\valbar - \val} \ge - g(\valbar) +
    \inner{A \wb(\val)}{\valbar} + f(\wb(\val)).
  \end{align*}
\end{lemma}
\begin{proof}
  Direct calculation by plugging in \eqref{eq:alpha_to_w} into
  \eqref{eq:dual} and using the convexity of $g$ yields
  \begin{align*}
    D(\val) + \inner{\grad D(\val)}{\valbar - \val} &= - g(\val) +
    \inner{A\wb(\val)}{\val} + f(\wb(\val)) + \inner{-\grad g(\val)
      + A \wb(\val)}{\valbar - \val} \\
    &\ge - g(\valbar) + \inner{A\wb(\val)}{\valbar} + f(\wb(\val)).
  \end{align*}
  \vspace{-2.6em}
  
\end{proof}

Furthermore, because $d$ is $\sigma$-strongly convex, it follows
that
\begin{align}
  \label{eq:bregman_ge_normsq}
  \Delta(\valbar, \val) = d(\valbar) - d(\val) - \inner{\grad
    d(\val)}{\valbar - \val} \ge \frac{\sigma}{2} \nbr{\valbar -
    \val}^{2}_{2}.
\end{align}
As $\val_{0}$ minimizes $d$ over $Q_{2}$, we have
\begin{align}
  \label{eq:opt_cond_d2}
  \inner{\grad d(\val_{0})}{\val - \val_{0}} \ge 0 \qquad \forall\ \val \in Q_2.
\end{align}

We first show that the initial $\wb_1$ and $\val_1$ satisfy the excessive gap condition \eqref{eq:excessive_gap_condition}.
Since $-D$ is $L$-\lcg, so
\begin{align*}
  D(\val_{1}) & \ge D(\val_{0}) + \inner{\grad D(\val_{0})}{\val_{1} -
    \val_{0}} - \frac{1}{2} L \nbr{\val_{1} - \val_{0}}^2 \\
  (\text{using defn. of } \mu_{1} \text{ and
    \eqref{eq:bregman_ge_normsq}}) & \ge D(\val_{0}) + \inner{\grad
    D(\val_{0})}{\val_{1} - \val_{0}} - \mu_{1} \Delta(\val_{1}, \val_{0}) \\
  (\text{using defn.\ of } \val_{1}) &= D(\val_{0}) - \mu_{1}
  \min_{\val \in Q_2} \cbr{-\frac{1}{\mu_{1}} \inner{ \grad
      D(\val_{0})}{\val - \val_{0}} + \Delta(\val, \val_{0})} \\
  (\text{using } \eqref{eq:opt_cond_d2} \text{ and } d(\val_{0}) = 0) &
  \ge D(\val_{0}) - \mu_{1} \min_{\val \in Q_2} \cbr{-\frac{1}{\mu_{1}}
    \inner{\grad D(\val_{0})}{\val - \val_{0}}
    + d(\val)} \\
  & = \max_{\val \in Q_2} \cbr{D(\val_{0}) + \inner{\grad
      D(\val_{0})}{\val - \val_{0}} - \mu_{1}\, d(\val)} \\
  (\text{using Lemma }\ref{lamma:nesterov:helper_alpha}) & \ge \max_{\val
    \in Q_2} \cbr{-g(\val) + \inner{A \wb(\val_{0})}{\val} +
    f(\wb(\val_{0})) - \mu_{1}\, d(\val)} \\
  & = J_{\mu_{1}} (\wb_1),
\end{align*}
which shows that our initialization indeed satisfies
\eqref{eq:excessive_gap_condition}. Second, we prove by induction that
the updates in Algorithm \ref{algo:nesterov05_bregman} maintain
\eqref{eq:excessive_gap_condition}.  We begin with two useful
observations. Using \eqref{eq:mu2-update} and the definition of
$\tau_{k}$, one can bound
\begin{align}
  \label{eq:tau-mu-bound}
  \mu_{k+1} = \frac{6}{(k+3)(k+2)}
  \frac{L}{\sigma} \geq \tau_{k}^{2} \frac{L}{\sigma}.
\end{align}
Let $\vbeta := \val_{\mu_{k}}(\wb_{k})$. The optimality conditions for
\eqref{eq:w_to_alpha} imply
\begin{align}
  \label{eq:opt_cond_alpha}
  \inner{\mu_{k} \grad d(\vbeta) - A \wb_{k} + \grad g(\vbeta)}{\val - \vbeta} \ge 0.
\end{align}

By using the update equation for $\wb_{k+1}$ and the convexity of $f$
\begin{align*}
J_{\mu_{k+1}}(\wb_{k+1}) & = f(\wb_{k+1}) + \max_{\val \in Q_2}
\cbr{\inner{A \wb_{k+1}}{\val} - g(\val) - \mu_{k+1}
  d(\val)} \\
& = f((1-\tau_{k}) \wb_{k} + \tau_{k} \wb(\valhat))  \\
& \; \; \; + \max_{\val \in Q_2} \left\{(1-\tau_{k})
\inner{A \wb_{k}}{\val} + \tau_{k} \inner{A
      \wb(\valhat)}{\val}- g(\val) - (1-\tau_{k})\mu_{k}
d(\val) \right\} \\
  & \le \, \, \max_{\val \in Q_2} \left\{ (1-\tau_{k}) T_{1} + \tau_{k}
T_{2} \right \},
\end{align*}
\[
  \text{where } T_{1}  = \sbr{-\mu_{k} d(\val) + \inner{A \wb_{k}}{\val} -
    g(\val) + f(\wb_{k})} \text{ and } T_{2} = \sbr{-g(\val) +
    \inner{A \wb(\valhat)}{\val} + f(\wb(\valhat))}.
\]
$T_{1}$ can be bounded as follows
\begin{align*}
  T_{1} & = -\mu_{k} d(\val) + \inner{A \wb_{k}}{\val} - g(\val) +
  f(\wb_{k}) \\
  (\text{using defn.\ of } \Delta) & = -\mu_{k} \cbr{\Delta(\val,
    \vbeta) + d(\vbeta) + \inner{\grad d(\vbeta)}{\val - \vbeta}} +
  \inner{A\wb_{k}}{\val} - g(\val) + f(\wb_{k}) \\
  (\text{using } \eqref{eq:opt_cond_alpha}) & \le -\mu_{k} \Delta(\val,
  \vbeta) - \mu_{k} d(\vbeta) + \inner{-A\wb_{k} + \grad g(\vbeta)}{\val
    - \vbeta} + \inner{A \wb_{k}}{\val} - g(\val) + f(\wb_{k}) \\
  &=  -\mu_{k} \Delta(\val, \vbeta) - \mu_{k} d(\vbeta) + \inner{A
    \wb_{k}}{\vbeta} - g(\val) + \inner{\grad g(\vbeta)}{\val - \vbeta}
  + f(\wb_{k}) \\
  (\text{using convexity of g}) & \le -\mu_{k} \Delta(\val, \vbeta) -
  \mu_{k} d(\vbeta) + \inner{A \wb_{k}}{\vbeta} - g(\vbeta) + f(\wb_{k})
  \\
  (\text{using defn.\ of } \vbeta) & =  -\mu_{k} \Delta(\val, \vbeta) +
  J_{\mu_{k}}(\wb_{k}) \\
  (\text{using induction assumption}) & \le  -\mu_{k} \Delta(\val, \vbeta) + D(\val_{k}) \\
  (\text{using concavity of } D) & \le -\mu_{k} \Delta(\val, \vbeta) +
  D(\valhat) + \inner{\grad D(\valhat)}{\val_{k} - \valhat},
\end{align*}
while $T_{2}$ can be bounded by using Lemma \ref{lamma:nesterov:helper_alpha}:
\begin{align*}
  T_{2} = -g(\val) + \inner{A \wb(\valhat)}{\val} + f(\wb(\valhat)) \le
  D(\valhat) + \inner{\grad D(\valhat)}{\val - \valhat}.
\end{align*}

Putting the upper bounds on $T_{1}$ and $T_{2}$ together, we obtain the
desired result.
\begin{align*}
  J_{\mu_{k+1}}(\wb_{k+1}) & \le \max_{\val \in Q_2} \left \{
    (1-\tau_{k}) \sbr{-\mu_{k} \Delta(\val, \vbeta) + D(\valhat) +
      \inner{\grad D(\valhat)}{\val_{k} - \valhat}} + \tau_{k}
    \sbr{D(\valhat) + \inner{\grad
        D(\valhat)}{\val - \valhat}}  \right\} \\
  & = \max_{\val \in Q_2} \cbr{-\mu_{k+1} \Delta(\val, \vbeta) +
    D(\valhat) + \inner{\grad D(\valhat)}{(1-\tau_{k})\val_{k} +
      \tau_{k} \val - \valhat}} \\
  (\text{using defn.\ of } \valhat)& = \max_{\val \in Q_2}
  \cbr{-\mu_{k+1} \Delta(\val, \vbeta) + D(\valhat) + \tau_{k}
    \inner{\grad D(\valhat)}{\val - \vbeta}} \\
  & = -\min_{\val \in Q_2} \cbr{\mu_{k+1} \Delta(\val, \vbeta) -
    D(\valhat) - \tau_{k} \inner{\grad D(\valhat)}{\val - \vbeta} } \\
  (\text{using defn.\ of } \valtil) & = -\mu_{k+1} \Delta(\valtil,
  \vbeta) + D(\valhat) + \tau_{k} \inner{\grad D(\valhat)}{\valtil -
    \vbeta} \\
  (\text{using } \eqref{eq:bregman_ge_normsq}) & \le
  -\smallfrac{1}{2}\mu_{k+1} \nbr{\valtil - \vbeta}^2 + D(\valhat) +
  \tau_{k} \inner{\grad D(\valhat)}{\valtil - \vbeta} \\
  (\text{using } \eqref{eq:tau-mu-bound})& \le
  -\smallfrac{1}{2}\tau_{k}^2 L \nbr{\valtil - \vbeta}^2 + D(\valhat) +
  \tau_{k} \inner{\grad D(\valhat)}{\valtil - \vbeta} \\
  (\text{using defn. of }\val_{k+1}) & = -\smallfrac{1}{2} L
  \nbr{\val_{k+1} - \valhat}^2 + D(\valhat) + \inner{\grad
    D(\valhat)}{\val_{k+1} -
    \valhat} \\
  (\text{by }L \text{-\lcg\ of } -D) & \le D(\val_{k+1}).
\end{align*}

\section{Proof of Corollary \ref{thm:rate_conv_dual_gap}}
\label{sec:app_dual_gap}

\begin{align*}
  D(\val_{k+1}) &\ge J_{\mu_{k+1}}(\wb_{k+1}) = f(\wb_{k+1}) +
  \max_{\val} \cbr{\inner{A\wb_{k+1}}{\val} - g(\val) - \mu_{k+1}
    d(\val)} \\
  &\ge f(\wb_{k+1}) + \inner{A\wb_{k+1}}{\val^*} -
  g(\val^*) - \mu_{k+1} d(\val^*) \\
  &\ge - g(\val^*) + \min_{\wb}
  \cbr{f(\wb) + \inner{A \wb}{\val^*}} - \mu_{k+1} d(\val^*) \\
  &= D(\val^*) - \mu_{k+1} d(\val^*).
\end{align*}

\section{Primal and Dual Objective Evaluation using Clique Decomposition}
\label{sec:primal_dual_eval}

We show how to efficiently compute the primal and dual objective
function values.  The primal objective value is easy due to the
convenience in computing $\nbr{\wb_k}^2$ and inner products between
$\wb_k$ and feature vectors.  Afterwards any MAP algorithm can be used
to find the $\max_{\yvec \in \Ycal}$. The dual objective \eqref{eq:dval-m3n} is also easy
since
\begin{align*}
  \sum_i \sum_{\yvec} \ell^i_{\yvec} (\val_k)^i_{\yvec}
  = \sum_i \sum_{\yvec} \sum_c \ell^i_{y_c} (\val_k)^i_{\yvec} =
  \sum_i \sum_c \sum_{y_c} \ell^i_{y_c}\sum_{\yvec: \yvec|_{c}= y_c}(\val_k)^i_{\yvec}
  = \sum_{i,c,y_c} \ell^i_{y_c} (\val_k)^i_{y_c},
\end{align*}
and the marginals of $\val_k$ are available. Finally, the quadratic
term in $D(\alphavec_k)$ can be computed as follows.
\begin{align*}
  \nbr{A^{\top} \alphavec_k}^2 = \nbr{\sum_{i, \yvec}
    \psivec^i_{\yvec} (\alphavec_k)^i_{\yvec}}^2 = \sum_c
  \nbr{\sum_{i, y_c} \psivec^i_{y_c} (\alphavec_k)^i_{ y_c}}^2 =
  \sum_c \sum_{i,j,y_c,y'_c} (\val_k)^i_{y_c} (\val_k)^j_{y'_c}
  k_c((\xvec^i,y_c), (\xvec^j, y'_c)),
\end{align*}
where the inner term is the same as the unnormalized expectation that
can be efficiently calculated. The last formula is only for nonlinear
kernels.

\section{Kernelizing the Excessive Gap Method for \mcn s}
\label{sec:app_kernel}

Compared with the linear kernel case, the only difficulty caused by nonlinear kernels is that the $\wb_k$ cannot be expressed explicitly.  However, if $\wb_k$ can be expressed as the expectation of the feature vector with respect to some distribution $\beta_k \in \Scal^n$, then we only need to update $\wb_k$ implicitly via $\beta_k$, and the inner product between $\wb_k$ and any feature vector can also be efficiently calculated.  We formalize and prove this claim by induction.

\begin{theorem}
  For all $k \ge 0$, there exists $\betavec_k \in \Scal^n$, such that $(\wb_k)_c = \frac{1}{\lambda} \FF[\psivec_c; \betavec_k]$, and $\betavec_k$ can be updated by
  \[
  \betavec_{k+1} = (1-\tau_k) \betavec_k + \tau_k \valhat_k.
  \]
\end{theorem}
\begin{proof}
First, $\wb_1 = \wb(\val_{0}) = \frac{1}{\lambda} \mathop{\oplus}\nolimits_{c \in \Ccal} \FF[\psivec_c; \val_0]$, so $\betavec_1 = \val_0$.  Suppose the claim holds for all $1, \ldots, k$, then
\begin{align*}
(\wb_{k+1})_c &= (1-\tau_k) (\wb_k)_c + \frac{\tau_k}{\lambda} \FF[\psivec_c; (\valhat_k)_c] = (1-\tau_k) \frac{1}{\lambda} \FF[\psivec_c; \betavec_k] + \frac{\tau_k}{\lambda} \FF[\psivec_c; (\valhat_k)_c] \\
&= \frac{1}{\lambda} \FF[\psivec_c; (1-\tau_k) (\betavec_k)_c + \tau_k (\valhat_k)_c].
\end{align*}
Therefore, we can set $\betavec_{k+1} = (1-\tau_k) \betavec_k + \tau_k \valhat_k \in \Scal^n$.
\end{proof}

In general $\valhat_k \neq \valtil_k$, hence $\betavec_k \neq \val_k$.  To compute $\inner{\psivec^i_{y_c}}{(\wb_k)_c}$ required by \eqref{eq:valtil-factor}, we have
\[
\inner{\psivec^i_{y_c}}{(\wb_k)_c} = \inner{\psivec^i_{y_c}}{\frac{1}{\lambda} \sum_j \sum_{y'_c} \beta^j_{y'_c}\psivec^j_{y'_c} } = \frac{1}{\lambda} \sum_j \sum_{y'_c} \beta^j_{y'_c} k_c((\xvec^i, y_c),(\xvec^j, y'_c)).
\]
And by using this trick, all the iterative updates in Algorithm
\ref{algo:nesterov05_bregman_struct} can be done efficiently.  So is the
evaluation of $\nbr{\wb_k}^2$ and the primal and dual objectives.  We
leave the details to the reader.

\newpage
\begin{center}
  {\Large \textbf{Supplementary Material}}
\end{center}

\vspace{2em}

\section{Proof of Lemma \ref{lem:dual-g-mud}}
\label{sec:proof_lem:dual-g-mud}

\begin{proof}
  Using \eqref{eq:m3n-gdef} and \eqref{eq:m3n-proxdef} we can write
  \begin{align*}
    (g + \mu d)^{\star} (\ub) & = \sup_{\val \in \Scal^{n}} \{
    \inner{\ub}{\val} - g(\val) - \mu d(\val) \}  \\
    & = \sup_{\val \in \Scal^{n}} \sum_{i} \sum_{\yb} u^i_{\yb}
    \alpha^{i}_{\yb} + \sum_{i} \sum_{\yb} \ell^{i}_{\yb} \alpha^{i}_{\yb} -
    \mu \sum_{i} \sum_{\yb} \alpha^{i}_{\yb} \log \alpha^{i}_{\yb} - \mu \log
    n - \mu \log
    \abr{\Ycal} \\
    & = \sup_{\val \in \Scal^{n}} \sum_{i} \sum_{\yb} (u^i_{\yb} +
    \ell^{i}_{\yb} - \mu \log \alpha^{i}_{\yb}) \alpha^{i}_{\yb} - \mu \log n
    - \mu \log \abr{\Ycal}
  \end{align*}
  By introducing non-negative Lagrange multipliers $\sigma_{i}$ we can
  write the partial Lagrangian of the above maximization problem:
  \begin{align*}
    L (\val, \sigma) = \sup_{\val \in \Scal^{n}} \sum_{i} \sum_{\yb}
    (u^i_{\yb} + \ell^{i}_{\yb} - \mu \log \alpha^{i}_{\yb}) \alpha^{i}_{\yb}
    - \mu \log n - \mu \log \abr{\Ycal} - \sum_i \sigma_i \rbr{\sum_{\yb}
      \alpha^{i}_{\yb}-\frac{1}{n}}.
  \end{align*}
  Taking partial derivative with respect $\alpha^{i}_{\yb}$ and setting it
  to 0, we get
  \begin{align*}
    u^i_{\yb} + \ell^{i}_{\yb} - \mu \log \alpha^{i}_{\yb} - \mu - \sigma_i = 0.
  \end{align*}
  Therefore
  \begin{align*}
    \alpha^{i}_{\yb} = \frac{\exp \rbr{\frac{u^i_{\yb} +
          \ell^{i}_{\yb}}{\mu}}}{n Z_i}, \qquad \text{where } Z_i :=
    \sum_{\yb} \exp \rbr{\frac{u^i_{\yb} + \ell^{i}_{\yb}}{\mu}}.
  \end{align*}
  Plugging this back to the Lagrangian, we can eliminate both $\val$ and
  $\sigma_{i}$ and write out the solution of the optimization problem in
  closed form
  \begin{align*}
    \sum_{i,\yb} (\mu \log Z_{i} + \mu \log n) \alpha^{i}_{\yb} - \mu \log n -
    \mu \log \abr{\Ycal} = \frac{\mu}{n} \sum_{i=1}^{n} \log \sum_{\yb \in
      \Ycal} \exp \rbr{\frac{u^i_{\yb} + \ell^{i}_{\yb}}{\mu}} - \mu \log
    \abr{\Ycal}.
  \end{align*}
\end{proof}

\end{document}